\documentclass[10pt,journal,cspaper,compsoc]{IEEEtran}
\usepackage{graphicx}
\usepackage{hyperref}
\usepackage{subfigure}
\usepackage{amsmath, amssymb}
\usepackage{multirow}
\usepackage{algorithm}
\usepackage{algorithmic}
\usepackage{url}
\usepackage{paralist,algorithmic,algorithm}
\usepackage{booktabs}
\usepackage{balance}

\begin{document}
\title{A Framework of Sparse Online Learning\\ and Its Applications}

\author{Dayong~Wang,
        Pengcheng~Wu,
        Peilin~Zhao,
        and~Steven~C.H.~Hoi, 
\IEEEcompsocitemizethanks{\IEEEcompsocthanksitem Dayong Wang is with Department of Computer Science and Engineering, Michigan State University, USA 48824.\protect\\
E-mail: dywang@msu.edu
\IEEEcompsocthanksitem Pengcheng Wu and Steven C.H. Hoi are with School of Information Systems, Singapore Management University, Singapore 178902.\protect\\
E-mail:  \{pcwu, chhoi\}@smu.edu.sg
\IEEEcompsocthanksitem $^\ast$Peilin Zhao, the corresponding author, is with Institute for Infocomm Research, A*STAR, Singapore 138632.\protect\\
E-mail:  zhaop@i2r.a-star.edu.sg}
\thanks{}}

\newtheorem{thm}{Theorem}
\newtheorem{prop}{Proposition}
\newtheorem{lemma}{Lemma}
\newtheorem{cor}[thm]{Corollary}
\newtheorem{definition}[thm]{Definition}
\newcommand{\includeMyGraphicA}[1]{\includegraphics[width=2.2in, height=4in]{#1}}
\newcommand{\includeMyGraphicB}[1]{\includegraphics[width=3.2in, height=2.6in]{#1}}
\newcommand{\makeMyboxA}[1]{\makebox[2.2in]{#1}}
\newcommand{\makeMyboxB}[1]{\makebox[3.2in]{#1}}
\def \g    {\mathbf{g}}
\def \x    {\mathbf{x}}
\def \z    {\mathbf{z}}
\def \hy   {\hat{y}}
\def \H    {\mathcal{H}_{\kappa}}
\def \R    {\mathbb{R}}
\def \w    {\mathbf{w}}
\def \u    {\mathbf{u}}
\def \E    {\mathbb{E}}
\def \M    {\mathcal{M}}
\def \L    {\mathcal{L}}
\def \I    {\mathbb{I}}
\def \sign {\mathrm{sign}}
\def \det  {\mathrm{det}}
\def \diag {\mathrm{diag}}
\newcommand{\fix}{\marginpar{FIX}}
\newcommand{\new}{\marginpar{NEW}}

\IEEEcompsoctitleabstractindextext{%
\begin{abstract}
The amount of data in our society has been exploding in the era of big data today. In this paper, we address several open challenges of big data stream classification, including high volume, high velocity, high dimensionality, high sparsity, and high class-imbalance. Many existing studies in data mining literature solve data stream classification tasks in a batch learning setting, which suffers from poor efficiency and scalability when dealing with big data. To overcome the limitations, this paper investigates an online learning framework for big data stream classification tasks. Unlike some existing online data stream classification techniques that are often based on first-order online learning, we propose a framework of Sparse Online Classification (SOC) for data stream classification, which includes some state-of-the-art first-order sparse online learning algorithms as special cases and allows us to derive a new effective second-order online learning algorithm for data stream classification. In addition, we also propose a new cost-sensitive sparse online learning algorithm by extending the framework with application to tackle online anomaly detection tasks where class distribution of data could be very imbalanced. We also analyze the theoretical bounds of the proposed method, and finally conduct an extensive set of experiments, in which encouraging results validate the efficacy of the proposed algorithms in comparison to a family of state-of-the-art techniques on a variety of data stream classification tasks.
\end{abstract}

}

\maketitle
\IEEEdisplaynotcompsoctitleabstractindextext
\IEEEpeerreviewmaketitle

\section{Introduction}
In the era of big data today, the amount of data in our society has been exploding, which has raised many opportunities and challenges for data analytic research in data mining community. In this work, we aim to address the challenging real-world big data stream classification task, such as web-scale spam email classification. In general, big data stream classification has several characteristics:

\begin{compactitem}
\item{\bf high volume}: one has to deal with huge amount of existing training data, in million or even billion scale;
\item{\bf high velocity}: new data often arrives sequentially and very rapidly, e.g., about $182.9$ billion emails are sent/received worldwide every day  according to an email statistic report by the Radicati Group~\cite{Radicati2013};
\item{\bf high dimensionality}: there are a large number of features, e.g., for some spam email classification tasks, the length of the vocabulary list can go up from $10,000$ to $50,000$ or even to million scale;
\item{\bf high sparsity}: many feature elements are zero, and the faction of active features is often small, e.g., the spam email classification study in \cite{youn2007spam} showed that accuracy saturates with dozens of features out of tens of thousands of features; and
\item{\bf high class-imbalance}: some class considerably dominates the others, e.g., for spam email classification tasks, the number of non-spam (ham) emails is often much larger than the number of spam emails.
\end{compactitem}

The above characteristics present huge challenges for big data stream classification tasks when using conventional data stream classification techniques that are often restricted to batch learning setting and thus suffer from several critical drawbacks: (i) it requires a large memory capacity for caching arrived examples; (ii) it is expensive to collect and train on the entire data set; (iii) it suffers from expensive re-training cost whenever new training data arrives; and (iv) their assumption that all training data must be available a prior does not hold for real-world data stream applications where data arrives rapidly in a sequential manner.

To tackle the above challenges, a promising approach is to explore online learning methodology that performs incremental training over streaming data in a sequential manner.  Typically, an online learning algorithm processes one instance at a time and makes very simple updates with each arriving example repeatedly. In contrast to batch learning algorithms, online algorithms are not only more efficient and scalable, but also able to avoid expensive re-training cost when handling new training data, making them more favorite choices for solving large-scale machine learning tasks towards big data stream applications. In literature, a large variety of algorithms have been proposed, including a number of first-order algorithms~\cite{rosenblatt1958perceptron, DBLP:journals/jmlr/CrammerDKSS06} and second-order algorithms~\cite{DBLP:journals/siamcomp/Cesa-BianchiCG05, DBLP:journals/jmlr/SchraudolphYG07, DBLP:journals/jmlr/BordesBG09}. Despite being studied extensively, traditional online-learning algorithms suffer from critical limitation for high-dimensional data. This is because they assume at least one weight for every feature and most of the learned weights are often nonzero, making them of low efficiency not only in computational time but also in memory cost for both training and test phases. Sparse online learning~\cite{langford-2009-sparse} aims to overcome this limitation by inducing sparsity in the weights learned by an online-learning algorithm.

In this paper, we introduce a framework of Sparse Online Learning for solving large-scale high-dimensional data stream classification tasks. We show that the proposed framework covers some existing first-order sparse online classification algorithm, and is able to further derive new algorithms by exploiting the second order information. The proposed sparse online classification scheme is far more efficient and scalable than the traditional batch learning algorithms for data stream classification tasks. We further give theoretical analysis of the proposed algorithm and conduct an extensive set of experiments. The empirical evaluation shows that the proposed algorithm could achieve state-of-the-art performance.  The rest of this paper is organized as follows. Section 2 reviews related work. Section 3 presents our problem formulation. Section 4 proposes our novel framework. Section 5 discusses our experimental results, and section 6 concludes this work.

As a summary, our main contributions include:
\begin{itemize}
\item We propose a general online learning framework, which can easily derive first order and second order algorithms.
\item We provide general theoretical analysis including general regret and mistake bounds for the proposed algorithms.
\item The proposed algorithms are evaluated on several high-dimensional large-scale benchmark databases, where the state-of-the-art performances are archived.
\end{itemize}

\section{Related Work}
Our work is closely related to the studies of online learning in machine learning and data mining. Below we briefly review some important related works.

\subsection{Online Learning}

Online learning represents a family of efficient and scalable machine learning algorithms~\cite{hoi2014libol}, which would online optimize some performance measure including, accuracy~\cite{DBLP:journals/jmlr/CrammerDKSS06}, AUC~\cite{DBLP:conf/icml/ZhaoHJY11}, cost-sensitive metrics~\cite{DBLP:journals/tkde/WangZH14}, etc. Unlike batch learning methods that suffer from expensive re-training cost, online learning works sequentially by performing highly efficient (typically constant) updates for each new training data, making it highly scalable for data stream classification. In literature, various techniques~\cite{bianchi-2006-prediction, rosenblatt1958perceptron, DBLP:journals/jmlr/CrammerDKSS06, DBLP:conf/icml/DredzeCP08, DBLP:conf/nips/CrammerKD09, DBLP:journals/jmlr/ZhaoHJ11, wang2012exact} have been proposed for online learning. The well-known first-order online learning algorithms include
Perceptron~\cite{rosenblatt1958perceptron,LMP99}, Passive-Aggressive (PA) algorithms~\cite{DBLP:journals/jmlr/CrammerDKSS06}, etc.

The most well-known method is the Perceptron algorithm~\cite{rosenblatt1958perceptron,LMP99}, which updates the model by adding a new example as a support vector with some constant weight. Recently, a series of sophisticated online learning algorithms have been proposed by following the criterion of maximum margin learning principle~\cite{DBLP:journals/jmlr/Gentile01, KivinenSW01, DBLP:journals/jmlr/CrammerDKSS06}. One famous algorithm is the Passive-Aggressive (PA) algorithm~\cite{DBLP:journals/jmlr/CrammerDKSS06}, which evolves a classifier by suffering less loss on the current instance without moving far from the previous function.

In recent years, the design of many efficient online learning algorithms has been influenced by convex optimization tools. Furthermore, it was observed that most previously proposed efficient online algorithms can be jointly analyzed based on the following elegant model~\cite{Shalev-Shwartz:2012:OLO:2185819.2185820}:
\begin{algorithm}[htpb]
\caption{Online Convex Optimization Scheme}\label{alg:online_convex_optimiaztion_scheme}
\begin{algorithmic}
\STATE\textbf{INPUT : A convex set $\R^d$}
\FOR{$t=1,\ldots, T$}
\STATE predict a vector $\w_t \in \R^d$;
\STATE receive a convex loss function $\ell_t : S \rightarrow \mathbb{R}$;
\STATE suffer loss $\ell_t(\w_t)$;
\ENDFOR
\end{algorithmic}
\end{algorithm}

Based on the previous framework, we can consider online learning as an algorithmic framework for convex online learning problem:
$$\min_{\w}{f(\w)} = \min_{\w}{\sum_{t}{\ell_t(\w)}},$$
where $f(\w)$ is a convex empirical loss function for the sum of losses over a sequence of observations. The regret of the algorithm is defined as follows:
$$R_{T} = \sum_{t=1}^{T}\ell_t(\w_t) - \min_{\w} \sum_{t=1}^{T}\ell_t(\w),$$
where $\w$ is any vector in the convex space $\R^d$. The goal of online learning algorithm is to find a low regret scheme, in which the regret $R_T$ grows sub-linearly with the number of iteration $T$. As a result, when the round number $T$ goes to infinity, the difference between the \emph{average} loss of the learner and the \emph{average} lost of the best learner tends to zero.

Although the general online learning algorithms (e.g., Perceptron and PA) have solid theoretical guarantees and performs well on many applications, generally they are limited in several aspects. First, the general online learning algorithms exploit the full features, which is not suitable for large-scale high-dimensional problem. To tackle this limitation, the \emph{sparse online learning} has been extensively studied recently. Second, the general online learning algorithms only exploit the first order information and all features are adopted the same learning rate. This problem can be addressed by \emph{second order online learning} algorithms. Last but not least, the general online learning algorithms are not suitable for the imbalance input data streams, which can be efficiently solved by the \emph{cost-sensitive online learning} algorithms. In the following parts, we will briefly introduce several representative algorithms in the previous three aspects.

\subsection{Sparse Online Learning}
\emph{Sparse online learning}~\cite{duchi-2009-sparse,langford-2009-sparse} aims to learn a sparse linear classifier, which only contains limited size of active features. It has been actively studied~\cite{duchi-2009-sparse,xiao2010dual,Shalev-Shwartz2011,wang2013online}. There are two group of solutions for \emph{sparse online learning}. The first group study on sparse online learning follows the general idea of subgradient descent with truncation. For example, Duchi and Singer propose the FOBOS algorithm~\cite{duchi-2009-sparse}, which extends the \emph{Forward-Backward Splitting} method to solve the sparse online learning problem in two phases: (i) an unconstrained subgradient descent step with respect to the loss function, and (ii) an instantaneous optimization for a trade-off between minimizing regularization term and keeping close to the result obtained in the first phase. The optimization problem in the second phase can be efficiently solved by adopting simple \emph{soft-thresholding} operations that perform some truncation on the weight vectors. Following the similar scheme, Langford et al.~\cite{langford-2009-sparse} argue that truncation on every iteration is too aggressive as each step modifies the
coefficients by only a small amount, and propose the \emph{Truncated Gradient} (TG) method which truncates coefficients every $K$ steps when they are less than a predefined threshold $\theta$. The second group study on sparse online learning mainly follows the dual averaging method of~\cite{nesterov2009primal}, can explicitly exploit the regularization structure in an online setting. For example, One representative work is \emph{Regularized Dual Averaging}(RDA)~\cite{xiao2010dual}, which learns the variables by solving a simple optimization problem that involves the running average of all past subgradients of the lost functions, not just the subgradient in each iteration. Lee et al.~\cite{lee2012manifold} further extends the RDA algorithm by using a more aggressive truncation threshold and generates significantly more sparse solutions.

\subsection{Second-order Online Learning}
\emph{Second Order Online Learning} aims to dynamically incorporate knowledge of observed data in earlier iteration to perform more informative gradient-based learning. Unlike first order algorithms that often adopt the same learning rate for all coordinates, the second order online learning algorithms adopt different distills to the step size employed for each coordinate.
A variety of second order online learning algorithms have been proposed recently. Some technique attempts to incorporate knowledge of the geometry of the data observed in earlier iterations to perform more effective online updates. For example, Balakrishnan et al.~\cite{balakrishnan2008algorithms} propose algorithms for sparse linear classifiers in the massive data setting, which requires $O(d^2)$ time and $O(d^2)$ space in the worst case. Another state-of-the-art technique for second order online learning is the family of confidence-weighted (CW) learning algorithms~\cite{DBLP:conf/icml/DredzeCP08, DBLP:conf/nips/CrammerDP08, crammer2009adaptive, ma2010exploiting, wang2012exact}, which exploit confidence of weights when making updates in online learning processes. In general, the second order algorithms are more accurate, converge faster, but fall short in two aspects (i) they incur higher computational cost especially when dealing with high-dimensional data; and (ii) the weight vectors learned are often not sparse, making them unsuitable for high-dimensional data. Recently, Duchi et al. address the sparsity and second order update in the same framework, and proposed the Adaptive Subgradient method~\cite{duchi2011adaptive} (Ada-RDA), which adaptively modifies the proximal function at each iteration to incorporate knowledge about geometry of the data.

\subsection{Cost-Sensitive Online Learning}
Cost-sensitive classification has been extensively studied in data mining and machine learning. In the past decade, a variety of cost-sensitive metrics have been proposed to tackle this problem. For example, the weighted sum of \emph{sensitivity} and \emph{specificity}~\cite{Brodersen:2010:BAP:1904935.1905533}, and the weighted \emph{misclassification cost}~\cite{conf/ecml/AkbaniKJ04, Elkan:2001:FCL:1642194.1642224}. Both cost-sensitive classification and online learning have been studied extensively in data mining and machine learning communities, respectively. There are only a few works on \emph{cost-sensitive online learning}. For example, Wang et al.~\cite{DBLP:journals/tkde/WangZH14} proposed a family of cost-sensitive online classification framework, which are designed to directly optimize two well-known cost-sensitive measures. Zhao and Hoi~\cite{Zhao:2013:COA:2487575.2487647} tackle the same problem by adopting the double updating technique and propose Cost-Sensitive Double Updating Online Learning (CSDUOL).

\section{Sparse Online Learning for Data Stream Classification}

In this section, we first introduce a general sparse online learning framework for online data stream classification, and then provide the theoretical analysis on the framework. The framework will be used to derive the family of first-order and second-order sparse online classification algorithms in the following section.

\subsection{General Sparse Online Learning}
Without loss of generality, we consider the sparse online learning algorithm for the binary classification problem, which is also mentioned as sparse online classification problem in this paper. The sparse online classification algorithm generally works in rounds. Specifically, at the round $t$, the algorithm is presented one instance $\x_t\in \R^d$, then the algorithm predicts its label  as
\begin{eqnarray}
\hy_t=\sign(\w_t^\top\x_t), \nonumber
\end{eqnarray}
where $\w_t\in \R^d$ is linear classifier maintained by the algorithm. After the prediction, the algorithm will receive the true label $y_t\in\{+1,-1\}$, and suffer a loss $\ell_t(\w_t)$. Then, the algorithm would update its prediction function $\w_t$ based on the newly received $(\x_t,y_t)$. The standard goal of online learning is to minimize the number of mistakes suffered by the online algorithm. To facilitate the analysis, we firstly introduce several functions. Firstly, the hinge loss $\ell_t(\w; (\x_t,y_t))=[1-y_t\w^\top\x_t]_+$, where $[a]_+=\max(a, 0)$, is the most popular loss function for binary classification problem. Given a series of $\delta$-strongly convex functions $\Phi_{t=1,\ldots,T}$, with respect to the norms $\|\cdot\|_{\Phi_t}$ and the dual norms $\|\cdot\|^*_{\Phi_t}$. The proposed general sparse online classification (SOC) algorithm is shown in Algorithm~\ref{alg:general_frame}.
\begin{algorithm}[htpb]
\caption{General Sparse Online Learning (SOL)}\label{alg:general_frame}
\begin{algorithmic}
\STATE\textbf{INPUT :$\lambda$, $\eta$}
\STATE\textbf{INITIALIZATION :} $\theta_1=0$.
\FOR{$t=1,\ldots, T$}
\STATE receive  $\x_t\in \R^d$;
\STATE $\u_t=\nabla \Phi^*_t(\theta_t)$;
\STATE $\w_t=\arg\min_{\w}\frac{1}{2}\|\u_t-\w\|^2_2+\lambda_t\|\w\|_1$;
\STATE predict $\hy_t=\sign(\w_t^\top\x_t)$;
\STATE receive $y_t$ and suffer $\ell_t(\w_t)=[1-y_t\w_t^\top\x_t]_+$;
\IF{$\ell_t(\w_t)>0$}
    \STATE $\theta_{t+1}=\theta_t-\eta_t \z_t$, where $\z_t=\nabla\ell_t(\w_t)$;
\ENDIF
\ENDFOR
\end{algorithmic}
\end{algorithm}

\subsection{Theoretical Bound Analysis}
In this section, we analysis the regret $R_T$ of the general sparse online learning (SOL) algorithm. Firstly, we will present a key lemma, which will facilitate the following analysis.
\begin{lemma}\label{lem:framework}
Let $\Phi_t, t=1,\ldots,T$ be $\delta$-strongly convex functions with respect to the norms $\|\cdot\|_{\Phi_t}$ and let $\|\cdot\|^*_{\Phi_t}$ be the respective dual norms. Let $\Phi(0)=0$, and $\x_1,\ldots,\x_T$ be an arbitrary sequence of vectors in $\R^d$. Assume that algorithm~\ref{alg:general_frame} is run on this sequence with the function $\Phi_t$, Then, we have the following inequality
\begin{eqnarray}\label{eqn:sparse-bound}
&&\hspace{-0.3in}\sum^T_{t=1}\eta_t (\w_t-\w)^\top\z_t\le\Phi_T(\w)\\
&&\hspace{-0.2in}+\sum^T_{t=1}\Big[\Phi^*_t(\theta_t)-\Phi^*_{t-1}(\theta_t)+\frac{\eta^2_t}{2\delta}\|\z_t\|^2_{\Phi^*_t}+\eta_t\lambda_t\|\z_t\|_1\Big],\nonumber
\end{eqnarray}
for any $\w$, and any $\lambda>0$.
\end{lemma}
\begin{proof}
Firstly, define $\Delta_t=\Phi^*_t(\theta_{t+1})-\Phi^*_{t-1}(\theta_t)$, then
\begin{eqnarray*}
\sum^T_{t=1}\Delta_t&=&\Phi^*_T(\theta_{T+1})-\Phi^*_0(\theta_1)=\Phi^*_T(\theta_{T+1})\\
 &\ge&\w^\top\theta_{T+1}-\Phi_T(\w),
\end{eqnarray*}
where the final inequality is due to Fenchel's inequality. In addition, we have
\begin{eqnarray*}
&&\hspace{-0.3in}\Delta_t=\Phi^*_t(\theta_{t+1})-\Phi^*_t(\theta_t)+\Phi^*_t(\theta_t)-\Phi^*_{t-1}(\theta_t)\\
&&\hspace{-0.3in}\le\Phi^*_t(\theta_t)-\Phi^*_{t-1}(\theta_t)-\eta_t (\nabla \Phi^*_t(\theta_t))^\top\z_t+\frac{\eta^2_t}{2\delta}\|\z_t\|^2_{\Phi^*_t}.
\end{eqnarray*}
Combining the above two inequalities, we get
\begin{eqnarray*}
&&-\sum^T_{t=1}\eta_t\w^\top\z_t -\Phi_T(\w)\le \sum^T_{t=1}\Delta_t\\
&&\le\sum^T_{t=1}[\Phi^*_t(\theta_t)-\Phi^*_{t-1}(\theta_t)-\eta_t \u_t^\top\z_t+\frac{\eta^2_t}{2\delta}\|\z_t\|^2_{\Phi^*_t}].
\end{eqnarray*}
Rearranging the above inequality, we get
\begin{eqnarray}\label{eqn:dense-bound}
&&\hspace{-0.3in}\sum^T_{t=1}\eta_t (\u_t-\w)^\top\z_t \nonumber\\
&&\hspace{-0.3in}\le\Phi_T(\w)+\sum^T_{t=1}[\Phi^*_t(\theta_t)-\Phi^*_{t-1}(\theta_t)+\frac{\eta^2_t}{2\delta}\|\z_t\|^2_{\Phi^*_t}].
\end{eqnarray}
Now, we would connect $\w_t^\top\x_t$ and $\u_t^\top\x_t$ as follows:
\begin{eqnarray*}
&&\hspace{-0.3in}\w_t^\top\z_t=\sum^d_{i=1}w_{t,i}z_{t,i}=\sum^d_{i=1}\sign(u_{t,i})[|u_{t,i}|-\lambda_t]_+ z_{t,i}\nonumber\\
&&\hspace{-0.3in}=\sum_{u_{t,i}z_{t,i}\ge 0}[|u_{t,i}|-\lambda_t]_+ |z_{t,i}|-\sum_{u_{t,i}z_{t,i}< 0}[|u_{t,i}|-\lambda_t]_+ |z_{t,i}|\nonumber\\
&&\hspace{-0.3in}\le\sum_{u_{t,i}z_{t,i}\ge 0}|u_{t,i}| |z_{t,i}|+\sum_{u_{t,i}z_{t,i}< 0}(-|u_{t,i}||z_{t,i}|+\lambda_t|z_{t,i}|)\nonumber\\
&&\hspace{-0.3in}\le\sum_{u_{t,i}z_{t,i}\ge 0}u_{t,i}z_{t,i}+\sum_{u_{t,i}z_{t,i}< 0}(u_{t,i}z_{t,i}+\lambda_t|z_{t,i}|)\nonumber\\
&&\hspace{-0.3in}\le\u_t^\top\z_t+\lambda_t\|\z_t\|_1.
\end{eqnarray*}
Plugging this inequality into inequality~(\ref{eqn:dense-bound}) will conclude the lemma.
\end{proof}
Given this general lemma, we would provide a general corollary, which could directly upper bound the regret suffered by this framework. To derive this kind of corollary, we only need to lower bound the left hand side of the inequality~(\ref{eqn:sparse-bound}) by using $\ell_t(\w_t)-\ell_t(\w)\le (\w_t-\w)^\top\z_t$, which is the property of convex function.
\begin{cor}\label{cor:regret}
Under the assumptions of Lemma 1, if we further assume $\ell$ is convex and $\eta_t=\eta$, then the regret $R_T=\sum^T_{t=1}\ell_t(\w_t)-\min_{\w}\sum^T_{t=1}\ell_t(\w) $ of the proposed framework~(\ref{alg:general_frame}) satisfies the following inequality
\begin{eqnarray}
R_T\le\frac{\Phi_T(\w)}{\eta}+\sum^T_{t=1}[\frac{\eta}{2\delta}\|\z_t\|^2_{\Phi^*_t}+\lambda_t\|\z_t\|_1]+\frac{\sum^T_{t=1}\Delta^*_t}{\eta},
\end{eqnarray}
where $\Delta^*_t=\Phi^*_t(\theta_t)-\Phi^*_{t-1}(\theta_t)$.
\end{cor}
Given this framework and these analysis, we would drive some specific algorithms and their regret bounds.

\section{Derived Algorithms}
In this section, we will first recover the RDA~\cite{xiao2010dual} algorithm and then derive algorithm utilizing the second order-information. In this section, we will adopt the hinge loss function and denote $\L=\{t|\ell_t(\w_t)>0\}$. We denote $L_t=\I_{(\ell_t(\w_t)>0)}$, where $\I_{v}$ is indicator function, $\I_{v}=1$ if $v$ is true, otherwise $\I_v=0$.

\subsection{First Order Algorithm }
Set $\Phi_t(\w)=\frac{1}{2}\|\w\|^2_2$, which is 1-strongly convex with respect to $\|\cdot\|_2$. And it is known that the dual norm of $\|\cdot\|_2$ is $\|\cdot\|_2$ itself, while $\Phi^*_t=\Phi_t$. Under these assumptions, we get the first order sparse online learning (FSOL) algorithm, which is the same with Regularized Dual Averaging (RDA) algorithm with soft 1-norm regularization~\cite{xiao2010dual}.
\begin{algorithm}[htpb]
\caption{First Order Sparse Online Learning (FSOL)}\label{alg:First Order}
\begin{algorithmic}
\STATE\textbf{INPUT :$\lambda$, $\eta$}
\STATE\textbf{INITIALIZATION :} $\theta_1=0$.
\FOR{$t=1,\ldots, T$}
\STATE receive  $\x_t\in \R^d$;
\STATE $\w_t=\sign(\theta_t)\odot[|\theta_t|-\lambda_t]_+$;
\STATE predict $\hy_t=\sign(\w_t^\top\x_t)$ and receive $y_t\in\{-1,1\}$;
\STATE suffer $\ell_t(\w_t)=[1-y_t\w_t^\top\x_t]_+$;
\STATE $\theta_{t+1}=\theta_t+ \eta L_t y_t \x_t$;
\ENDFOR
\end{algorithmic}
\end{algorithm}

\begin{thm}
Let $(\x_1,y_1),\ldots,(\x_T,y_T)$ be a sequence of examples, where $\x_t\in \R^d$, $y_t\in\{-1,+1\}$ and $\|\x_t\|_1\le X$ for all $t$. If further set $\lambda_t=\eta\lambda$, then the regret $R_T=\sum^T_{t=1}\ell_t(\w_t)-\min_{\w}\sum^T_{t=1}\ell_t(\w) $ suffered by the algorithm~(\ref{alg:First Order}) is bounded as follows:
\begin{eqnarray*}
R_T\le\frac{\frac{1}{2}\|\w\|^2_2}{\eta}+\frac{\eta}{2}\sum^T_{t=1}X^2+\sum^T_{t=1}\eta\lambda X,
\end{eqnarray*}
for any $\w\in\R^d$. Further setting $\eta=\frac{\|\w\|_2}{\sqrt{ (X^2+2\lambda X)T}}$, we could have
\begin{eqnarray*}
R_T\le D\sqrt{ (X^2+2\lambda X)T},
\end{eqnarray*}
for any $\w\in\{\w\ |\|\w\|_2\le D\}$.
\end{thm}
\begin{proof}
Firstly $\Delta^*_t=\Phi^*_t(\theta_t)-\Phi^*_{t-1}(\theta_t)=0$, then according to corollary~(\ref{cor:regret}), we have
\begin{eqnarray*}
R_T&\le&\frac{\frac{1}{2}\|\w\|^2_2}{\eta}+\sum^T_{t=1}[\frac{\eta}{2}\|L_t y_t\x_t\|^2_2+\lambda_t\|L_t y_t\x_t\|_1]\\
&\le&\frac{\frac{1}{2}\|\w\|^2_2}{\eta}+\frac{\eta}{2}\sum^T_{t=1}X^2+\sum^T_{t=1}\eta\lambda X.
\end{eqnarray*}
\end{proof}
{\bf Remark:} This bound indicates the regret of this algorithm is upper bounded by $O(\sqrt{T})$, which recovers the results in~\cite{xiao2010dual}.

\subsection{Second Order Algorithm}
Set $\Phi_t(\w)=\frac{1}{2}\w^\top A_t\w$, where $A_t=A_{t-1}+\frac{\x_t\x_t^\top}{r}, r>0$ and $A_0=I$. It is easy to verify that $\Phi_t$ is 1-strongly convex with respect to $\|\w\|^2_{\Phi_t}=\w^\top A_t \w$. Its dual function $\Phi^*_t(\w)$ is $\frac{1}{2}\w^\top A^{-1}_t\w$, while $\|\w\|^2_{\Phi^*_t}=\w^\top A^{-1}_t\w$. Using the Woodbury identity, we can incrementally update the inverse of $A_t$ as $A^{-1}_t=A^{-1}_{t-1}-\frac{A^{-1}_{t-1}\x_t\x_t^\top A^{-1}_{t-1}}{r+\x^\top_t A^{-1}_{t-1}\x_t}$. Under these assumptions, we get the second order sparse online learning (SSOL) algorithm.
\begin{algorithm}[htpb]
\caption{Second Order Sparse Online Learning (SSOL)}\label{alg:Second Order}
\begin{algorithmic}
\STATE\textbf{INPUT :$\lambda$, $\eta$}
\STATE\textbf{INITIALIZATION :} $\theta_1=0$.
\FOR{$t=1,\ldots, T$}
\STATE receive  $\x_t\in \R^d$;
\STATE $A^{-1}_t=A^{-1}_{t-1}-\frac{A^{-1}_{t-1}\x_t\x_t^\top A^{-1}_{t-1}}{r+\x^\top_t A^{-1}_{t-1}\x_t}$;
\STATE $\u_t=A^{-1}_t \theta_t$;
\STATE $\w_t=\sign(\u_t)\odot[|\u_t|-\lambda_t]_+$;
\STATE predict $\hy_t=\sign(\w_t^\top\x_t)$ and receive $y_t\in\{-1,1\}$;
\STATE suffer $\ell_t(\w_t)=[1-y_t\w_t^\top\x_t]_+$;
\STATE $\theta_{t+1}=\theta_t+ \eta L_t y_t \x_t$;
\ENDFOR
\end{algorithmic}
\end{algorithm}

\begin{thm}\label{thm:ssol}
Let $(\x_1,y_1),\ldots,(\x_T,y_T)$ be a sequence of examples, where $\x_t\in \R^d$, $y_t\in\{-1,+1\}$ and $\|\x_t\|_1\le X$ for all $t$. If further set $\lambda_t=\lambda/t$, then the regret $R_T=\sum^T_{t=1}\ell_t(\w_t)-\min_{\w}\sum^T_{t=1}\ell_t(\w) $ suffered by the algorithm~(\ref{alg:Second Order}) is bounded as
\begin{eqnarray}
&&\hspace{-0.3in}R_T \nonumber\le\frac{D^2}{2\eta}+\frac{\eta}{2}rd\log((1+\frac{X^2}{r}T))+\lambda X[\log(T)+1],
\end{eqnarray}
for any $\w\in\{\w\ |\w^\top A_T\w\le D^2\}$.
\end{thm}
\begin{proof}
Firstly, it is easy to observe
\begin{eqnarray*}
\Delta^*_t&=&\frac{1}{2}\theta_t^\top A^{-1}_t\theta_t-\frac{1}{2}\theta_t^\top A^{-1}_{t-1}\theta_t\\
&&=-\frac{(\x_t^\top A^{-1}_{t-1}\theta_t)^2}{2(r+\x^\top_t A^{-1}_{t-1}\x_t)}\le 0.
\end{eqnarray*}
Then according to the conclusion in the corollary~(\ref{cor:regret}), we have
\begin{eqnarray*}
&&\hspace{-0.3in}R_T\le\frac{\w^\top A_T\w}{2\eta}+\sum^T_{t=1}[\frac{\eta}{2}L_t\x_t^\top A^{-1}_t \x_t+\lambda_t\|L_t y_t \x_t\|_1]\\
&&\hspace{-0.2in}\le\frac{\w^\top A_T\w}{2\eta}+\frac{\eta}{2}\sum^T_{t=1}\x_t^\top A^{-1}_t \x_t+X\sum^T_{t=1}\lambda_t\\
&&\hspace{-0.2in}\le\frac{\w^\top A_T\w}{2\eta}+\frac{\eta}{2}\sum^T_{t=1}\x_t^\top A^{-1}_t \x_t+\lambda X[\log(T)+1],
\end{eqnarray*}
where the final inequality used  $\sum^T_{t=1}\frac{1}{t}\le [\log(T)+1]$. Secondly, the second term of the right hand side can be upper bounded as
\begin{eqnarray*}
&&\hspace{-0.4in}\sum^T_{t=1}\x_t^\top A^{-1}_t \x_t=r\sum^T_{t=1}(1-\frac{\det(A_{t-1})}{\det(A_t)})\nonumber\\
&&\hspace{-0.2in}\le- r\sum^T_{t=1}\log(\frac{\det(A_{t-1})}{\det(A_t)})=r\log(\det(A_T)).
\end{eqnarray*}
Combining the above two inequalities gives
\begin{eqnarray}\label{eqn:temp-bound}
&&\hspace{-0.3in}R_T \nonumber\le\frac{\w^\top A_T\w}{2\eta}\\
&&\hspace{-0.0in}+\frac{\eta}{2}r\log(\det(A_T))+\lambda X[\log(T)+1].
\end{eqnarray}
Since $A_T=I+\sum^T_{t=1}\frac{\x_t\x_t^\top}{r}$, its eigenvalue $\mu_i$ satisfies
\begin{eqnarray*}
\mu_i\le 1+ trace(\sum^T_{t=1}\frac{\x_t\x_t^\top}{r})=1+ \sum^T_{t=1}\frac{\|\x_t\|^2_2}{r}.
\end{eqnarray*}
As a result, we have
\begin{eqnarray*}
\det(A_T)=\prod^d_{i=1}\mu_i\le (1+\frac{X^2}{r}T)^d.
\end{eqnarray*}
Plugging the above inequality into~(\ref{eqn:temp-bound}) will concludes this theorem.
\end{proof}
{\bf Remark:} According to this theorem, adopt the second order information for the sparse online learning does further minimize the regret bound to an order of $O(\log(T))$.

\subsection{Diagonal Algorithm}
Although the previous second order algorithm significantly reduced the regret bound than the first order algorithm, it will consume $O(d^2)$ time, which limits its application to real-world high dimension problems. To keep the computational time still $O(d)$ similar with the traditional online learning, we further explored its diagonal version, which will only maintain a diagonal matrix. Its details are in the Algorithm~(\ref{alg:diagonal}).
\begin{algorithm}[htpb]
\caption{Diagonal Second Order Sparse Online Learning}\label{alg:diagonal}
\begin{algorithmic}
\STATE\textbf{INPUT :$\lambda$, $\eta$}
\STATE\textbf{INITIALIZATION :} $\theta_1=0$.
\FOR{$t=1,\ldots, T$}
\STATE receive  $\x_t\in \R^d$;
\STATE $A^{-1}_t=A^{-1}_{t-1}-\frac{A^{-1}_{t-1}\diag(\x_t\x_t^\top) A^{-1}_{t-1}}{r+\x^\top_t A^{-1}_{t-1}\x_t}$;
\STATE $\u_t=A^{-1}_t \theta_t$;
\STATE $\w_t=\sign(\u_t)\odot[|\u_t|-\lambda_t]_+$;
\STATE predict $\hy_t=\sign(\w_t^\top\x_t)$ and receive $y_t\in\{-1,1\}$;
\STATE suffer $\ell_t(\w_t)=[1-y_t\w_t^\top\x_t]_+$;
\STATE $\theta_{t+1}=\theta_t+ \eta L_t y_t \x_t$;
\ENDFOR
\end{algorithmic}
\end{algorithm}

In the following experiments, we mainly adopt the diagonal second order sparse online learning algorithm unless otherwise specified, which is also denoted as ``SSOL".

\subsection{Cost-Sensitive Algorithm}
For the previous algorithms, the classifier is cost-insensitive, which suffers the same \emph{cost/lost} when the positive samples and the negative samples are misclassified. It is inappropriate for many data stream classification tasks in real-world applications, such as online anomaly detection, where the class distribution is often highly imbalanced. In this section, we propose a cost-sensitive sparse online classification algorithm by extending the sparse online learning framework for online anomaly detection tasks. Without loss of generality, we assume the positive class is the rare class in a set of streaming data, which contains more positive examples than negative samples. We will prefer a high \emph{cost/lost} value when a positive sample is misclassified, while a small \emph{cost/lost} value when a negative sample is misclassified.

Specifically, we respectively denote the number of positive samples and negative sample by $T_+$ and $T_-$; and $M_+, M_-$ are the number of false negative and false positive, respectively. We denote $T = T_+ + T_-$ and $M = M_+ + M_-$. Instead of using the cost-insensitive metric $accuracy = \frac{T - M}{T}$, researchers have proposed a variety of cost-sensitive metrics. One well-know cost-sensitive metric is the weighted \emph{sum} of $sensitivity = \frac{T_+ - M_+}{T_+}$ and $specificity=\frac{T_- - M_-}{T_-}$, which is defined as follows:
$$ sum = \mu_+ \frac{T_+ - M_+}{T_+} + \mu_- \frac{T_- - M_-}{T_-},$$
where $\mu_+ + \mu_- =1$ and $0 \leq \mu_+,\mu_- \leq 1$ are two parameters to trade off between sensitivity and specificity. In general, the higher the \emph{sum} value, the better the classification performance. Notably, when $\mu_+ = \mu_- = 0.5$, the corresponding sum is the well known balanced accuracy~\cite{Brodersen:2010:BAP:1904935.1905533}.

In general, the higher the \emph{sum} value, the better the classification performance. To maximize the sum value, based on the previous framework, we propose a cost-sensitive sparse online classification algorithm following the theoretical analysis in~\cite{DBLP:journals/tkde/WangZH14, Zhao:2013:COA:2487575.2487647}. In particular, we adopt a modified hinge loss function:
$$ (\rho  \I_{y_t=1} + \I_{y_t=-1})  [1 - y_t \w^\top \x_t]_{+},$$
where $\rho = \frac{\mu_+ T_-}{\mu_- T_+}$ and $\I_{v}$ is an indicator function, which $\I_{v}=1$ if $v$ is true, otherwise $\I_{v}=0$. In our experiment, we use the balance accuracy as the metric and set $\mu_+ = \mu_- = 0.5$. Generally, it is difficult to predict the number of positive and negative samples $T_+$ and $T_-$ in advance. So a more realistic setting is to use two weight parameters $c_+$ and $c_-$ for the positive and negative losses, respectively. Hence, the loss function is reformulated as:
$$(c_+  \I_{y_t=1} + c_-  \I_{y_t=-1})  [1 - y_t \w^\top \x_t]_{+}.$$
Denoting $c_t = c_+  \I_{y_t=1} + c_-  \I_{y_t=-1}$, the modified regret is derived as follow:
$$R_T = \sum_t c_t \ell_t(w_t) - \sum_t c_t \ell_t(w),$$
where $\ell_t(\w_t)=[1-y_t\w_t^\top\x_t]_+$.

Based on the proposed sparse online learning framework and the cost-sensitive loss function, we can achieve the cost-sensitive first order sparse online learning algorithm (CS-FSOL) shown in Algorithm~\ref{alg:cs-fsol}.
\begin{algorithm}[htpb]
\caption{Cost-Sensitive First Order Sparse Online Learning (CS-FSOL)}\label{alg:cs-fsol}
\begin{algorithmic}
\STATE\textbf{INPUT : $\lambda$, $\eta$, $c_{+1}$, $c_{-1}$ }
\STATE\textbf{INITIALIZATION :} $\theta_1=0$, $A_0^{-1} = I$.
\FOR{$t=1,\ldots, T$}
\STATE receive  $\x_t\in \R^d$;
\STATE $\w_t=\sign(\theta_t)\odot[|\theta_t|-\lambda_t]_+$;
\STATE predict $\hy_t=\sign(\w_t^\top\x_t)$ and receive $y_t\in\{-1,1\}$;
\STATE suffer $\ell_t(\w_t)=[1-y_t\w_t^\top\x_t]_+$;
\STATE $\theta_{t+1}=\theta_t+ \eta c_{y_t} L_t y_t \x_t$;
\ENDFOR
\end{algorithmic}
\end{algorithm}

For this algorithm, it is easy to observe that if we treat $\eta c_{y_t}$ as $\eta_t$, then it is the special case of the proposed framework~(\ref{alg:general_frame}) with $\Phi_t(\w)=\frac{1}{2}\|\w\|^2_2$. So, we would like to prove a new corollary for  the proposed framework~(\ref{alg:general_frame}), under the situation that $\eta_t = \eta c_{y_t}$. This can be achieved  by combining Lemma~\ref{lem:framework} with   $\eta c_{y_t}[\ell_t(\w_t)-\ell_t(\w)]\le \eta_t (\w_t-\w)^\top\z_t$. Specifically, we have the following corollary:

\begin{cor}\label{cor:regret-csol}
Under the assumptions of Lemma 1, if we further assume $\ell$ is convex and $\eta_t=\eta c_{y_t}$, then the regret $R_T=\sum^T_{t=1}c_{y_t}\ell_t(\w_t)-\min_{\w}\sum^T_{t=1}c_{y_t}\ell_t(\w) $ of the proposed framework~(\ref{alg:general_frame}) satisfies the following inequality
\begin{eqnarray}
&&\hspace{-0.4in}R_T\le\nonumber\\
&&\hspace{-0.4in}\frac{\Phi_T(\w)}{\eta}+\sum^T_{t=1}[\frac{\eta}{2\delta}\|c_{y_t}\z_t\|^2_{\Phi^*_t}+\lambda_t\|c_{y_t}\z_t\|_1]+\frac{\sum^T_{t=1}\Delta^*_t}{\eta},
\end{eqnarray}
where $\Delta^*_t=\Phi^*_t(\theta_t)-\Phi^*_{t-1}(\theta_t)$.
\end{cor}

Given the above corollary, we can prove  the following theorem for Algorithm~\ref{alg:cs-fsol}.

\begin{thm}
Let $(\x_1,y_1),\ldots,(\x_T,y_T)$ be a sequence of examples, where $\x_t\in \R^d$, $y_t\in\{-1,+1\}$ and $\|\x_t\|_1\le X$ for all $t$. If further set $\lambda_t=\eta\lambda$, then the regret $R_T=\sum^T_{t=1}c_{y_t}\ell_t(\w_t)-\min_{\w}\sum^T_{t=1}c_{y_t}\ell_t(\w) $ suffered by the algorithm~(\ref{alg:cs-fsol}) is bounded as follows:
\begin{eqnarray*}
R_T\le\frac{\frac{1}{2}\|\w\|^2_2}{\eta}+\frac{\eta}{2}\sum^T_{t=1}c_{y_t}X^2+\sum^T_{t=1}\eta\lambda c_{y_t} X,
\end{eqnarray*}
for any $\w\in\R^d$.

Further setting $\eta=\frac{\|\w\|_2}{\sqrt{ (X^2+2\lambda X)(T_+c_++T_-c_-)}}$, we could have
\begin{eqnarray*}
R_T\le D\sqrt{ (X^2+2\lambda X)(T_+c_++T_-c_-)},
\end{eqnarray*}
for any $\w\in\{\w\ |\|\w\|_2\le D\}$.
\end{thm}
We omit the proof, since it is easy.

In addition, we can also get the cost-sensitive  second order sparse online classification (CS-SSOL) algorithm shown in Algorithm~\ref{alg:cs-ssol}. However, its time complexity and space complexity are relatively high for high dimension datasets, we will only use its diagonal variant in practice, where only a diagonal $A_t^{-1}$ is maintained and updated.

\begin{algorithm}[htpb]
\caption{Cost-Sensitive Second Order Sparse Online Learning (CS-SSOL)}\label{alg:cs-ssol}
\begin{algorithmic}
\STATE\textbf{INPUT : $\lambda$, $\eta$, $c_{+1}$, $c_{-1}$ }
\STATE\textbf{INITIALIZATION :} $\theta_1=0$, $A_0^{-1} = I$.
\FOR{$t=1,\ldots, T$}
\STATE receive  $\x_t\in \R^d$;
\STATE $A^{-1}_t=A^{-1}_{t-1}-\frac{A^{-1}_{t-1}\x_t\x_t^\top A^{-1}_{t-1}}{r+\x^\top_t A^{-1}_{t-1}\x_t}$;
\STATE $\u_t=A^{-1}_t \theta_t$;
\STATE $\w_t=\sign(\u_t)\odot[|\u_t|-\lambda_t]_+$;
\STATE predict $\hy_t=\sign(\w_t^\top\x_t)$ and receive $y_t\in\{-1,1\}$;
\STATE suffer $\ell_t(\w_t)=[1-y_t\w_t^\top\x_t]_+$;
\STATE $\theta_{t+1}=\theta_t + \eta c_{y_t} L_t y_t \x_t$;
\ENDFOR
\end{algorithmic}
\end{algorithm}

It is easy to verify that this algorithm is the special case of the proposed framework~(\ref{alg:general_frame}), when $\eta_t=\eta c_{y_t}$ and $\Phi_t(\w)=\frac{1}{2}\w^\top A_t\w$. So, the corollary~\ref{cor:regret-csol} holds for this algorithm. Using this corollary, we can prove  the following theorem for Algorithm~\ref{alg:cs-ssol}.

\begin{thm}
Let $(\x_1,y_1),\ldots,(\x_T,y_T)$ be a sequence of examples, where $\x_t\in \R^d$, $y_t\in\{-1,+1\}$ and $\|\x_t\|_1\le X$ for all $t$. If further set $\lambda_t=\lambda/t$, then the regret $R_T=\sum^T_{t=1}c_{y_t}\ell_t(\w_t)-\min_{\w}\sum^T_{t=1}c_{y_t}\ell_t(\w) $ suffered by the algorithm~(\ref{alg:Second Order}) is bounded as
\begin{eqnarray}
&&\hspace{-0.3in}R_T \nonumber\le\frac{D^2}{2\eta}+c_{max}\frac{\eta}{2}rd\log(1+\frac{X^2}{r}T)+c_{max}\lambda X[\log(T)+1],
\end{eqnarray}
for any $\w\in\{\w\ |\w^\top A_T\w\le D^2\}$, where $c_{max}=\max(c_+,c_-)$.
\end{thm}
The proof of this theorem is omitted, since it is easy and can mainly follows the one for Theorem~\ref{thm:ssol}.

\section{Experiments}

In this section, we conduct an extensive set of experiments to evaluate the performance of the proposed sparse online classification algorithms on both synthetic and real datasets.

\subsection{Experimental Setup}

In our experiments, we compare the proposed algorithms with a set of state-of-the-art algorithms, including the sparse online learning algorithms and the cost-sensitive online learning algorithms. The methodology details of these algorithms are listed in Table~\ref{tab:algs}. The three existing algorithms (CS-OGD, CPA and PAUM) are cost-sensitive online learning without sparsity regularizer.

\begin{table*}[htb]
\caption{List of Compared Algorithms.}\label{tab:algs}
\centering
\begin{tabular}{llll}
\toprule
Algorithm & 1st/2nd Order     & Sparsity              & Description                            \\ \midrule
STG       & First Order     & Truncate Gradient     & Stochastic Gradient Descent~\cite{langford-2009-sparse}   \\
FOBOS     & First Order     & Truncate Gradient     & FOrward Backward Splitting~\cite{duchi-2009-sparse}    \\
Ada-FOBOS & Second Order    & Truncate Gradient     & Adaptive regularized FOBOS~\cite{duchi2011adaptive}    \\
Ada-RDA   & Second Order    & Dual Averaging        & Adaptive regularized RDA~\cite{duchi2011adaptive}      \\
FSOL      & First Order     & Dual Averaging        & The proposed Algorithm~\ref{alg:First Order}\\
SSOL      & Second Order    & Dual Averaging        & The proposed Algorithm~\ref{alg:diagonal}  \\ \midrule
CS-OGD     & First Order     & Non-Sparse            & Cost-Sensitive Online Gradient Descent~\cite{DBLP:journals/tkde/WangZH14} \\
CPA       & First Order     & Non-Sparse            & Cost-Sensitive Passive-Aggressive~\cite{DBLP:journals/jmlr/CrammerDKSS06} \\
PAUM      & First Order     & Non-Sparse            & Cost-Sensitive Perceptron Algorithm with Uneven Margin~\cite{li2002perceptron} \\
CS-FSOL   & First Order     & Dual Averaging        & The proposed Algorithm~\ref{alg:cs-fsol}\\
CS-SSOL   & Second Order    & Dual Averaging        & The proposed Algorithm~\ref{alg:cs-ssol}\\
\bottomrule
\end{tabular}
\end{table*}

To examine the binary classification performance, beside the synthetic dataset, we evaluate all the previous algorithms on a number of benchmark datasets from web machine learning repositories. Table~\ref{tab:datasets} shows the details of all the datasets in our experiments. These datasets are selected to allow us evaluate the algorithms on various characteristics of data, in which the number of training examples ranges from thousands to millions, feature dimensionality ranges from hundreds to about 16-million, and the total number of non-zero features on some dataset is more than one billion. For the very large-scale WEBSPAM dataset, we run the algorithms only once. The sparsity as shown in the last column of the table denotes the ratio of non-active feature dimensions, as some feature dimensions are never active in the training process, which is often the case for some real-world high-dimensional dataset, such as WEBSPAM.

\begin{table*}[htbp]
\centering
\caption{List of real-world datasets in our experiments.}\label{tab:datasets}
\begin{tabular}{@{}llrrrrrl@{}}
\toprule
DataSet     & Balance & \#Train    & \#Test   & \#Feature Dimension &  \#Nonzero Features & Sparsity(\%)  & $T_+ \setminus T_-$  \\ \midrule
AUT         & True    & 40,000     & 22,581   & 20,707              &  1,969,407         & 3.07          & $1 \setminus 0.33$\\
PCMAC       & True    & 1,000      & 946      & 7,510               &  55,470            & 3.99          & $1 \setminus 1.00$\\
NEWS        & True    & 10,000     & 9,996    & 1,355,191           &  5,513,533         & 29.88         & $1 \setminus 1.50$\\
RCV1        & True    & 781,265    & 23,149   & 47,152              &  59,155,144        & 8.80          & $1 \setminus 1.11$\\
URL         & True    & 2,000,000  & 396,130  & 3,231,961           &  231,249,028       & 7.44          & $1 \setminus 2.02$\\
WEBSPAM     & True    & 300,000    & 50,000   & 16,071,971          &  1,118,027,721     & 95.82         & $1 \setminus 0.64$\\
URL2        & False   & 1,000,000  & 100,000  & 3,231,961           &  114,852,082       & 44.96         & $1 \setminus 99$\\
WEBSPAM2    & False   & 100,000    & 10,000   & 16,071,971          &  224,201,808       & 96.19         & $1 \setminus 99$\\
\bottomrule
\end{tabular}
\end{table*}

We conduct experiments by following standard online learning settings for training a classifier, where an online learner receives a single training example at each iteration and updates the model sequentially. We will examine how different sparsity levels affect test error rate of the classifier trained from a single pass through the training data. Besides, we also measure time cost of different algorithms to evaluate the computational efficiency. To make a fair comparison, all the algorithms adopt the same experimental settings. We use hinge loss as the loss function for the applicable algorithms. To identify the best set of parameters, for each algorithm on each dataset, we conduct a 5-fold cross validation for grid searching the parameters with the fixed sparsity regularization parameter $\lambda=0$. In particular, the learning rates are searched from $2^{-1}$ to $2^9$ and the other parameters are searched from $2^{-5}$ to $2^5$. With the best tuned parameters, each algorithm is evaluated for $5$ times with a random permutation of a train set. All the experiments were conducted on a Linux server (with Intel Xeon CPU E5-2620 @2.00GHz, 4 CPU cores,  8GB memory) and the programming environment is based on C++ implementation compiled by g++.

\subsection{Experiment on Synthetic Dataset}

To evaluate if the proposed sparse online learning algorithm is able to identify effective features for learning the models, we design the first experiment on a synthetic dataset, which allows us to control the exact numbers of {\it effective/noisy} feature dimensions. In particular, we generate a synthetic dataset with high dimensionality and high sparsity by following the similar scheme in~\cite{DBLP:conf/nips/CrammerDP08, crammer2009adaptive}, which contains a set of \emph{effective} feature dimensions that are correlated with the class labels and a set of \emph{noisy} feature dimensions that are uncorrelated with the labels.

Specifically, we generate the synthetic dataset with $100,000$ training examples and $10,000$ test examples in $\mathbb{R}^{1000}$. For each example, the first $100$ dimensions are drawn from a multivariate Gaussian distribution with diagonal covariance. Each dimension of the mean vector is uniformly sampled from $-1$ to $1$, and each dimension of covariance is uniformly sampled from $0.5$ to $100$. We generate the split plane the same as the mean vector. To introduce noisy feature dimensions, we randomly choose $200$ noise dimensions out of the rest 900 dimensions for each example. Noises are drawn from a Gaussian distribution of $\mathcal{N}(0,100)$.

\begin{figure}[htb]
\centering
\includegraphics[width=3.5in]{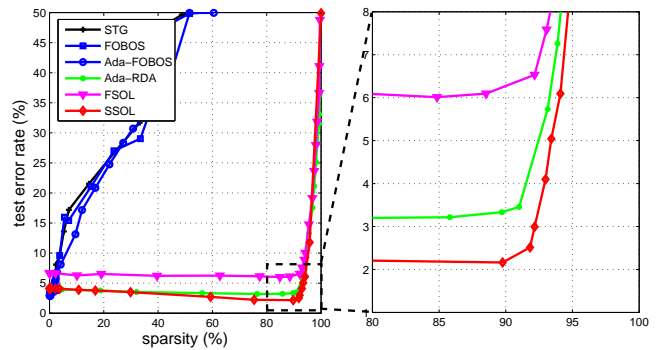}
\caption{Test error rate of sparse online classification on synthetic dataset.}\label{fig:synthetic_exp}
\end{figure}

We evaluate all the cost-insensitive sparse online classification algorithms on the synthetic dataset. Figure~\ref{fig:synthetic_exp} shows the test error rates of all the compared algorithms, where the right diagram is a sub-figure of the left one with sparsity from $80\%$ to $100\%$. Several observations can be drawn from the experimental results.

First of all, we observe that the test error rates of the \emph{truncate gradient} based algorithms (STG, FOBOS, Ada-FOBOS) decrease significantly when the sparsity level increases. By contrast,
for the \emph{dual averaging} based algorithms (FSOL, Ada-RDA, SSOL), the test error rates keep stable or even decrease when the sparsity level increases; But the test error rate increases dramatically when the sparsity level is higher than $90\%$---the actual sparsity level used for generating the synthetic data. The result indicates that the dual averaging based algorithms more effectively exploit the sparsity in the dataset. Similar observation was also reported in~\cite{xiao2010dual} who argued that the dual averaging based methods take more aggressive truncations and thus can generate significantly more sparse solutions. Second, the proposed second-order algorithm SSOL achieves the lowest error rate among all the compared algorithms, especially for high sparsity level. This observation can be seen more clearly in the right diagram of Figure~\ref{fig:synthetic_exp}.
The above encouraging experimental results indicate that the proposed SSOL algorithm can effectively exploit the sparsity for solving the sparse online classification tasks.

\subsection{Test Error Rate on Large Real Datasets} \label{sec:large_dataset}
\begin{figure*}[htbp]
\begin{center}
\includeMyGraphicA{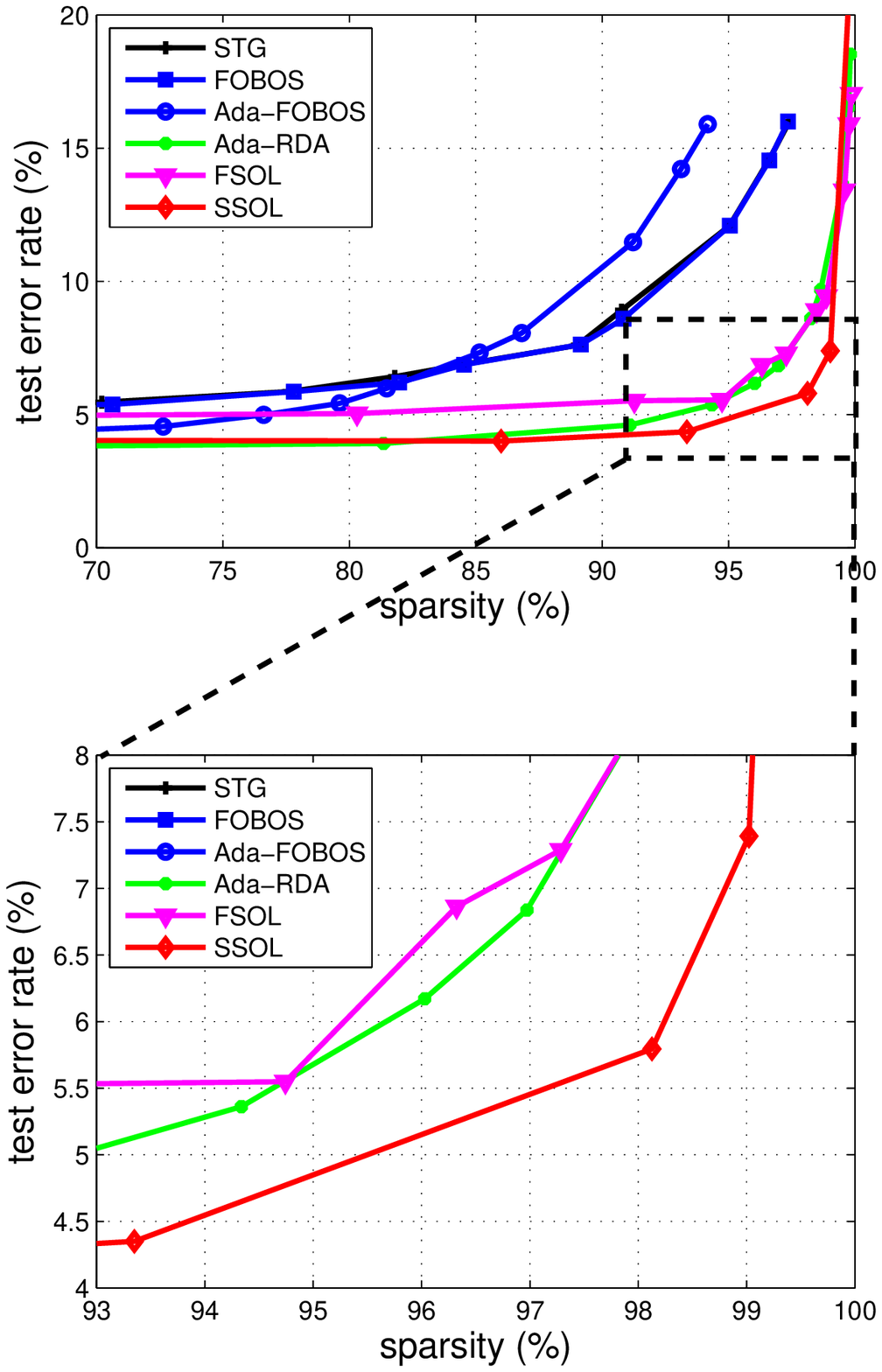}
\includeMyGraphicA{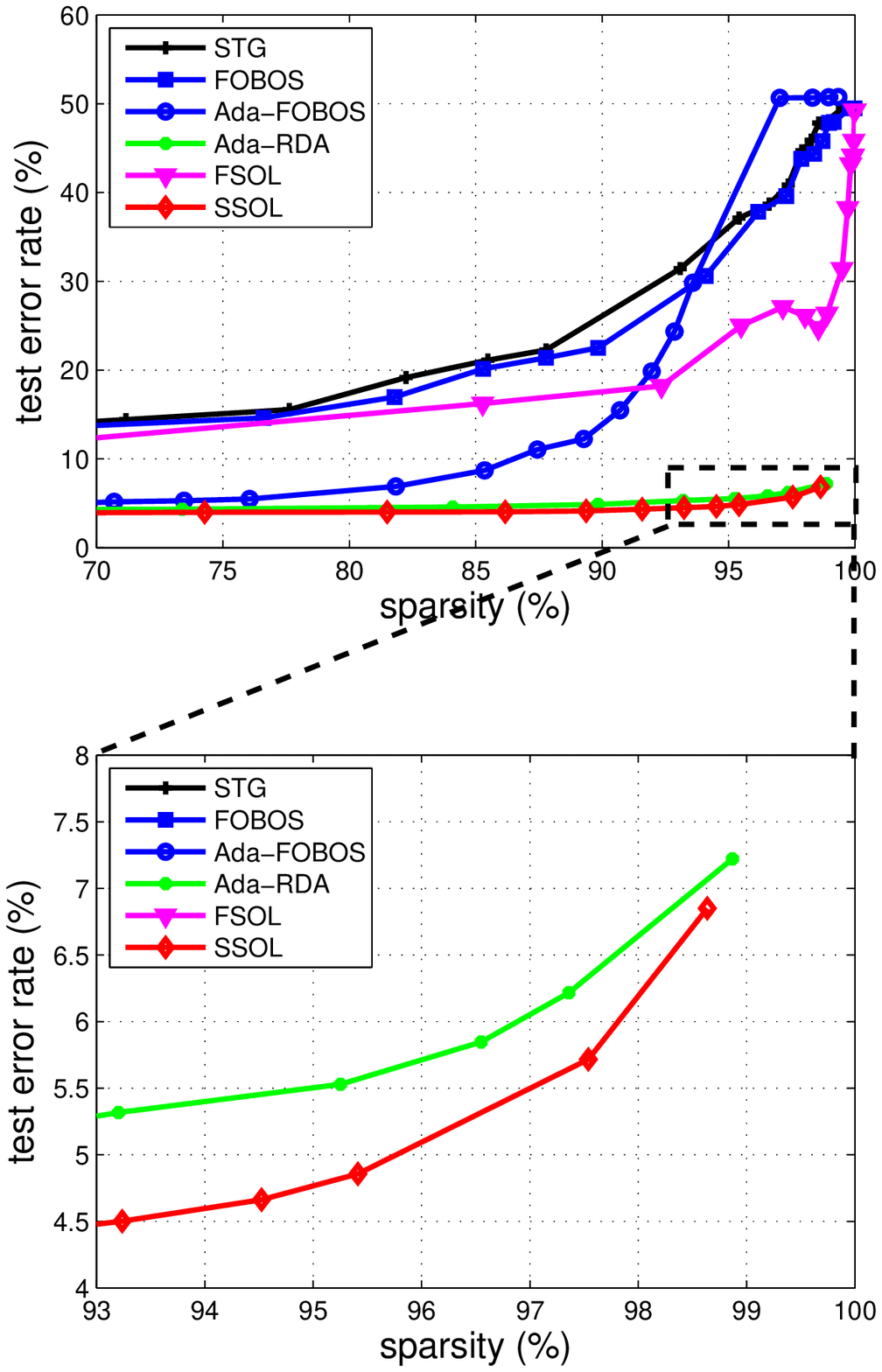}
\includeMyGraphicA{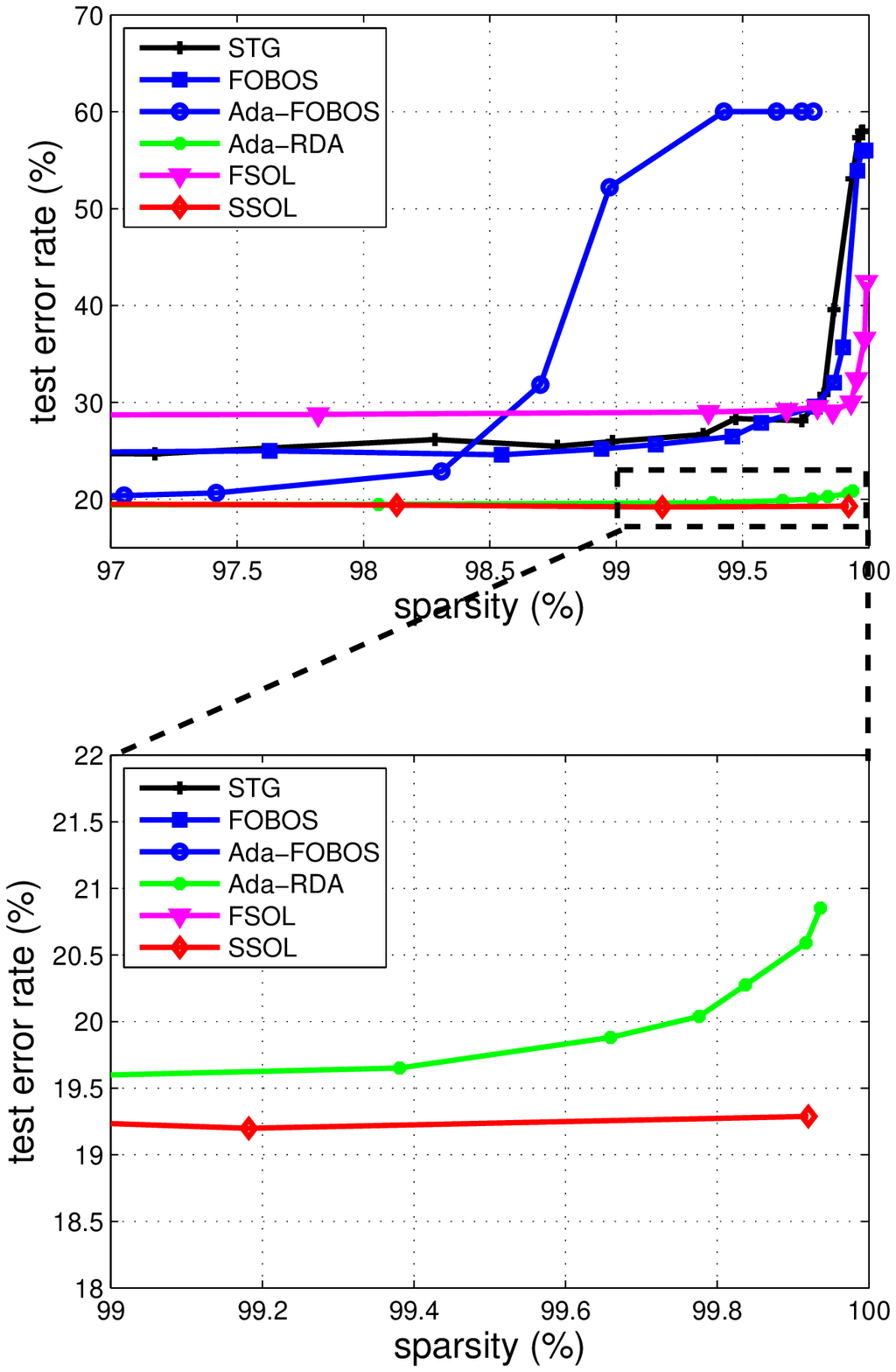} \\
\makeMyboxA{(a) AUT}
\makeMyboxA{(b) PCMAC}
\makeMyboxA{(c) NEWS} \\
\includeMyGraphicA{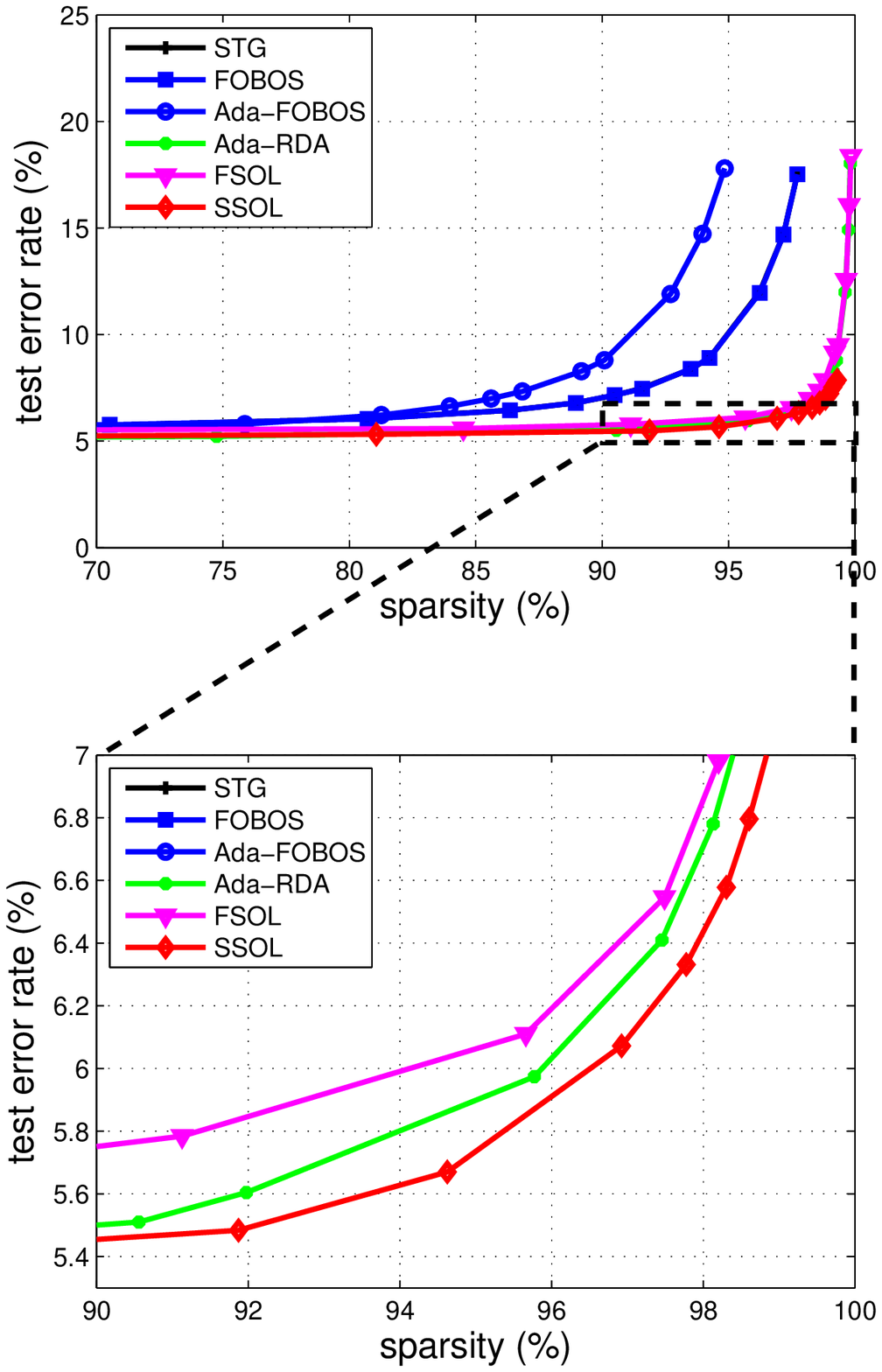}
\includeMyGraphicA{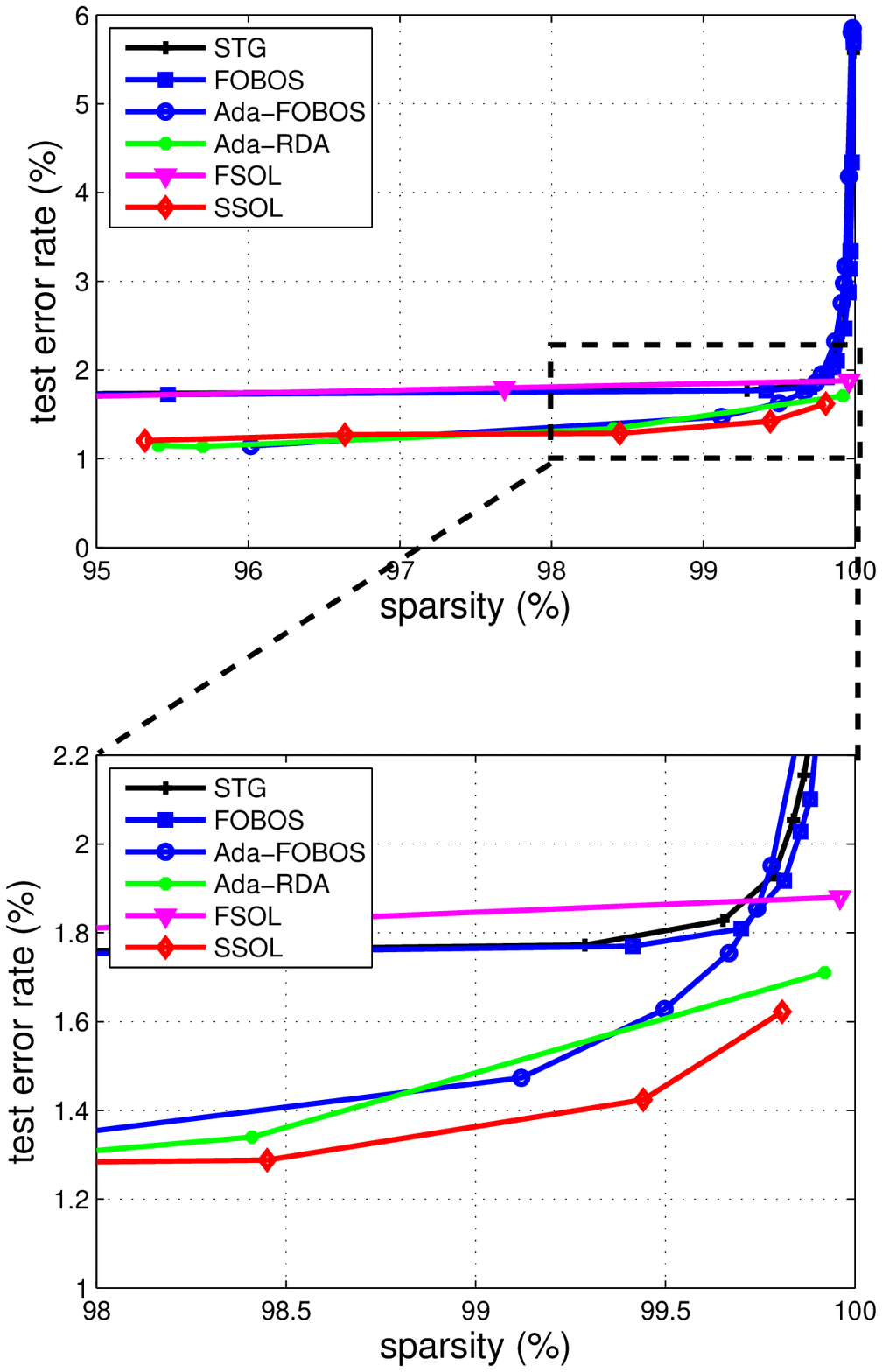}
\includeMyGraphicA{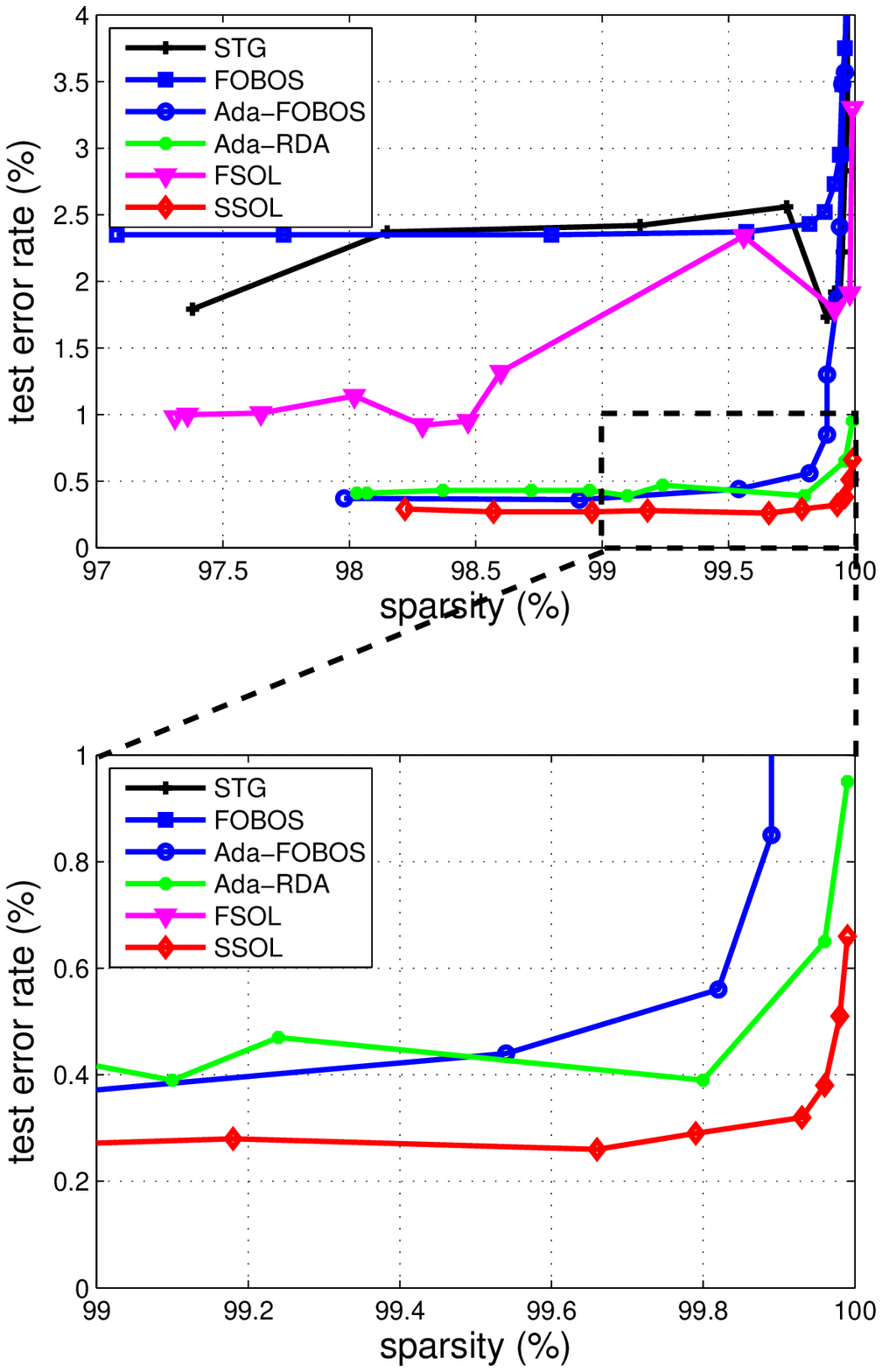} \\
\makeMyboxA{(d) RCV1}
\makeMyboxA{(e) URL}
\makeMyboxA{(f) WEBSPAM}
\caption{Test error rate on $6$ large real datasets. (a)-(b) are two general datasets, (c)-(f) are four large-scale high-dimensional sparse datasets. The second and forth rows are the sub-figures of the first and the third rows with high sparsity level, respectively.}\label{fig:err_rate}
\end{center}
\end{figure*}

In this experiment, we compare the proposed algorithms (FSOL and SSOL) with the other cost-insensitive algorithms on several real-world datasets. Table~\ref{tab:datasets} shows the details of six datasets, which can be roughly grouped into two major categories: the first two datasets (AUT and PCMAC) are general binary small-scale datasets and the corresponding experimental results are shown in Figure~\ref{fig:err_rate} (a)-(b); and the rest four datasets (NEWS, RCV1, URL, and WEBSPAM) are large-scale high-dimensional sparse datasets and the corresponding experimental results are shown in Figure~\ref{fig:err_rate} (c)-(f). We can draw several observation from these results as follows.

First of all, we observe that most algorithms can learn an effective sparse classification model with only marginal or even no loss of accuracy. For example, in Figure~\ref{fig:err_rate} (d), the performances of all the algorithms are almost stable when sparsity level is smaller than $80\%$. It indicates that all the compared sparse online classification algorithm can effectively explore the low level sparsity information.

Second, for most cases, we observe that there exists some sparsity threshold for each algorithm, below which test error rate does not change much; but when sparsity level is greater than the threshold, test error rate gets worse quickly.

Third, we observe that the dual averaging based second order algorithms (Ada-RDA and SSOL) consistently outperform the other algorithms (STG, FOBOS, FSOL, and Ada-FOBOS), especially for high sparsity level. This indicates that the dual averaging technique and second order updating rules are effective to boost the classification performance.

Finally, when the sparsity is high, an essential requirement for high-dimensional data stream classification tasks, the proposed SSOL algorithm consistently outperforms the other algorithms over all the evaluated datasets. For example, when the sparsity is about $99.8\%$ for the WEBSPAM dataset (the total feature dimensionality is $16,609,143$), the test error rate of SSOL is about $0.3\%$, while the Ada-RDA is $0.4\%$ and the Ada-FOBOS is $0.55\%$, as shown in Figure~\ref{fig:err_rate} (f).

\subsection{Running Time on Large Real Datasets}

We also examine time costs of different sparse online classification algorithms, and the experiment results are shown in Figure~\ref{fig:running_time}. In this experiment, we only adopt the four high-dimensional large-scale dataset. Several observations can be drawn from the results.

First of all, we observe that when the sparsity level is low, the time costs are generally stable; on the other hand, when the sparsity level is high, the time cost of the second other algorithms sometimes will somewhat increase. For example, the test costs of Ada-FOBOS, Ada-RDA and FSOL in Figure~\ref{fig:running_time} (b) \& (d). One possible reason may be that when the sparsity level is high, the model might not be informative enough for prediction and thus may suffer significant more updates. Since second-order algorithms are more complicated than first-order algorithms, they are more sensitive to the increasing number of updates.

Second, we can see that the proposed SSOL algorithm runs more efficiently than another second-order based algorithms (Ada-RDA and Ada-FOBOS). It is even sometimes better than the first order based algorithm (e.g. FOBOS and STD). However, the first order FSOL algorithm is consistently faster than the second order SSOL algorithm.

In summary, from the above analysis, we found that the proposed SSOL algorithm is able to achieve the comparable or even better accuracy of existing second-order algorithms, but has the comparably small time cost as state-of-the-art first-order algorithms with truncated gradient methods.

\begin{figure*}[htb]
\centering
\includeMyGraphicB{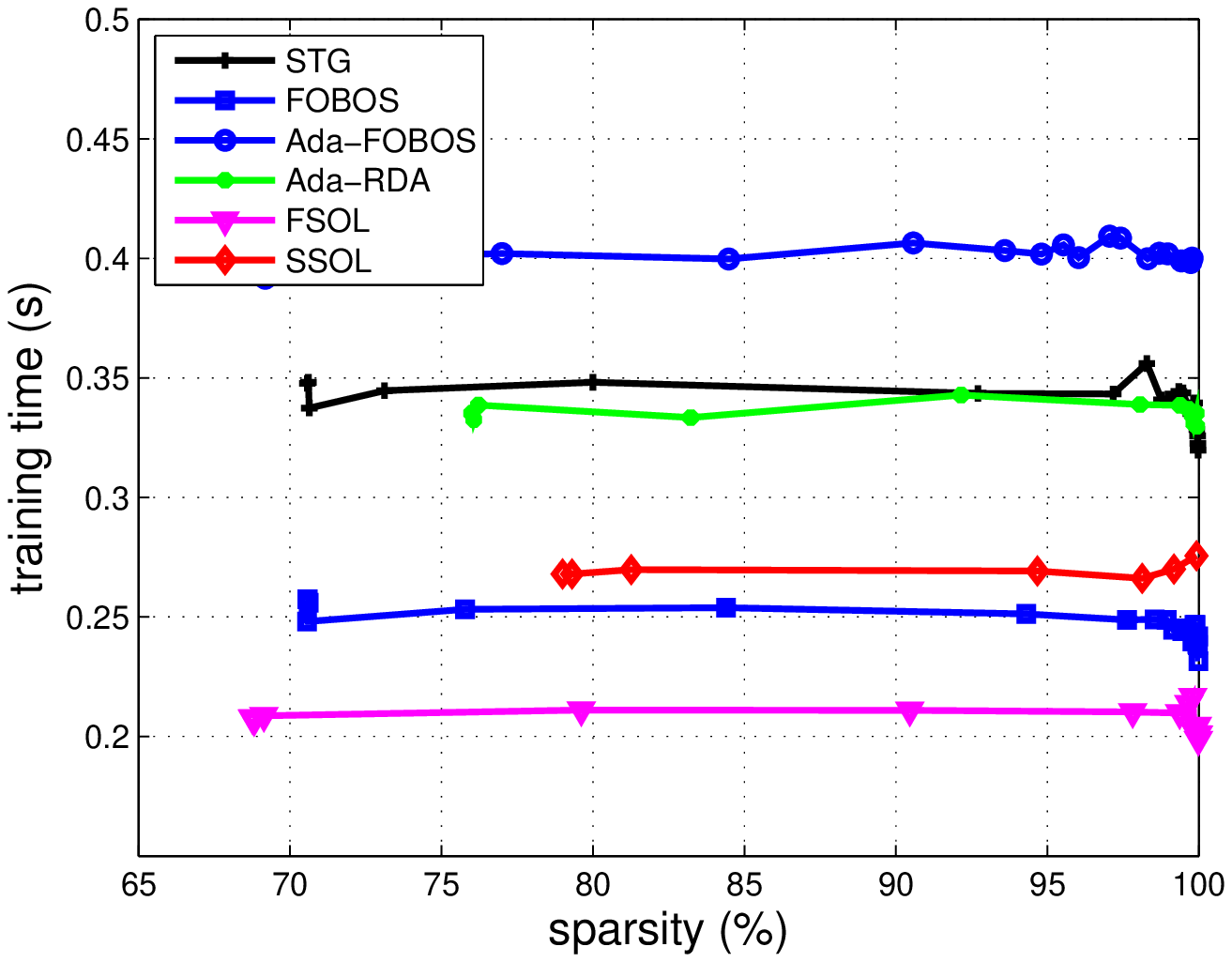}
\includeMyGraphicB{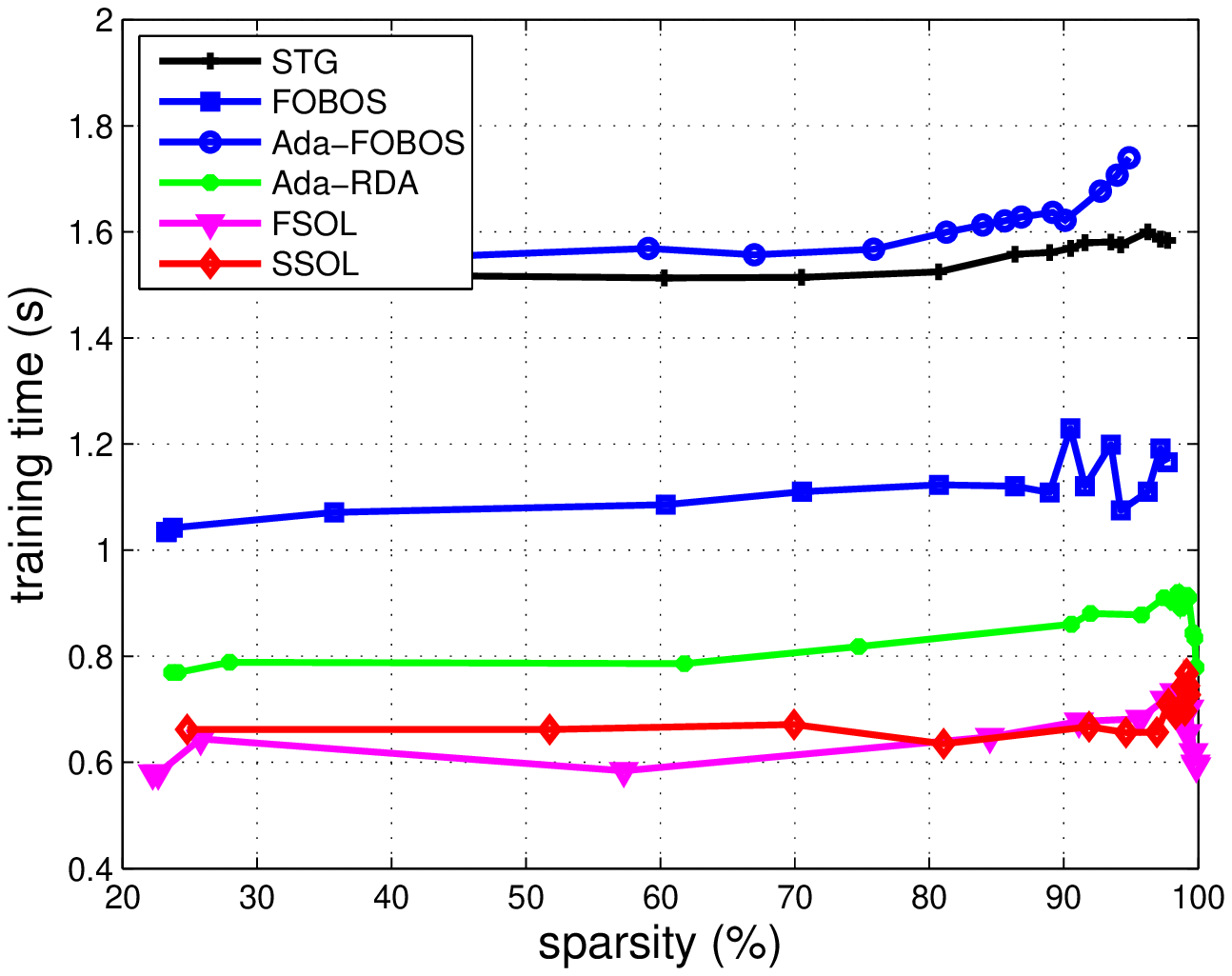} \\
\makeMyboxB{(a) NEWS}
\makeMyboxB{(b) RCV1} \\
\includeMyGraphicB{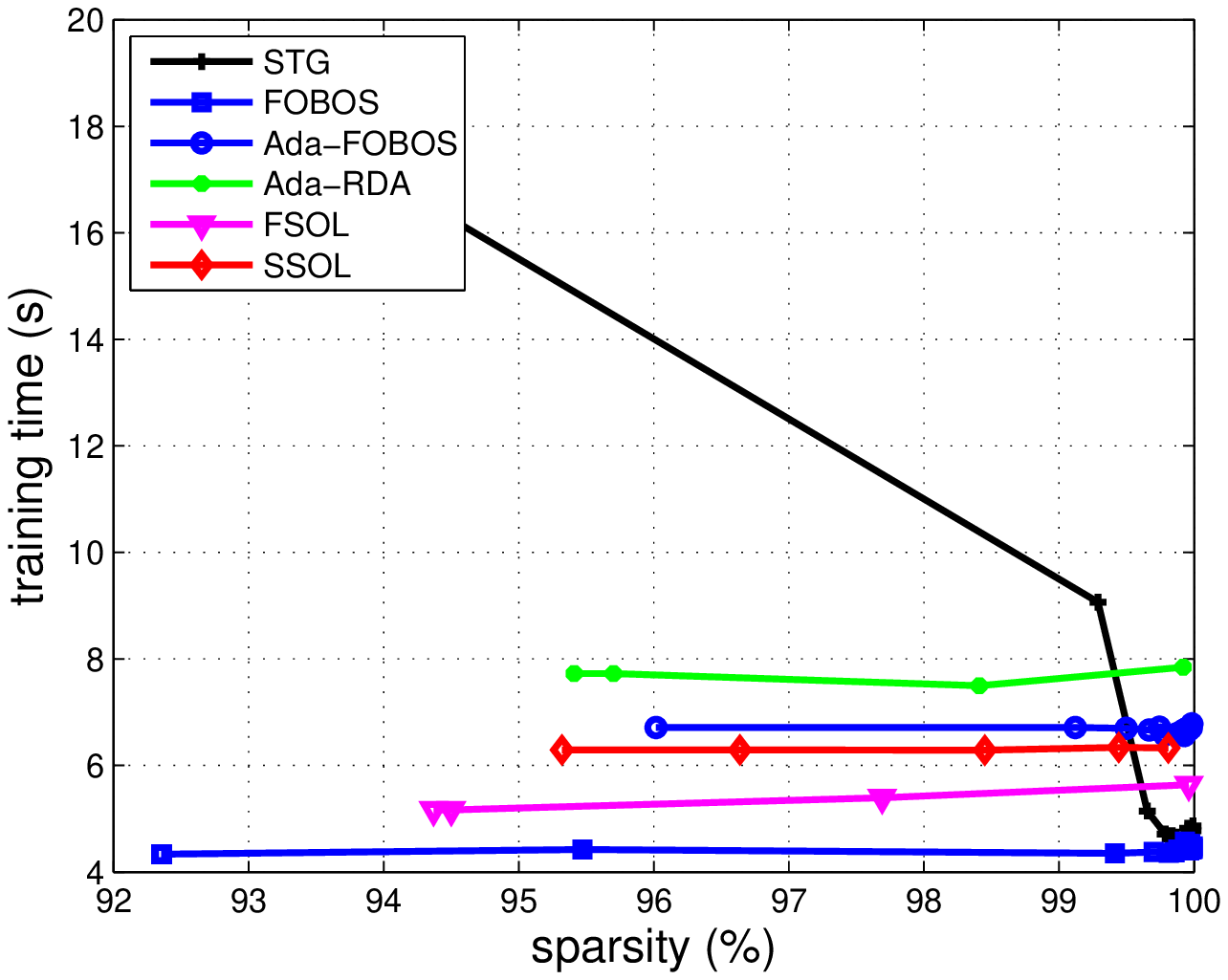}
\includeMyGraphicB{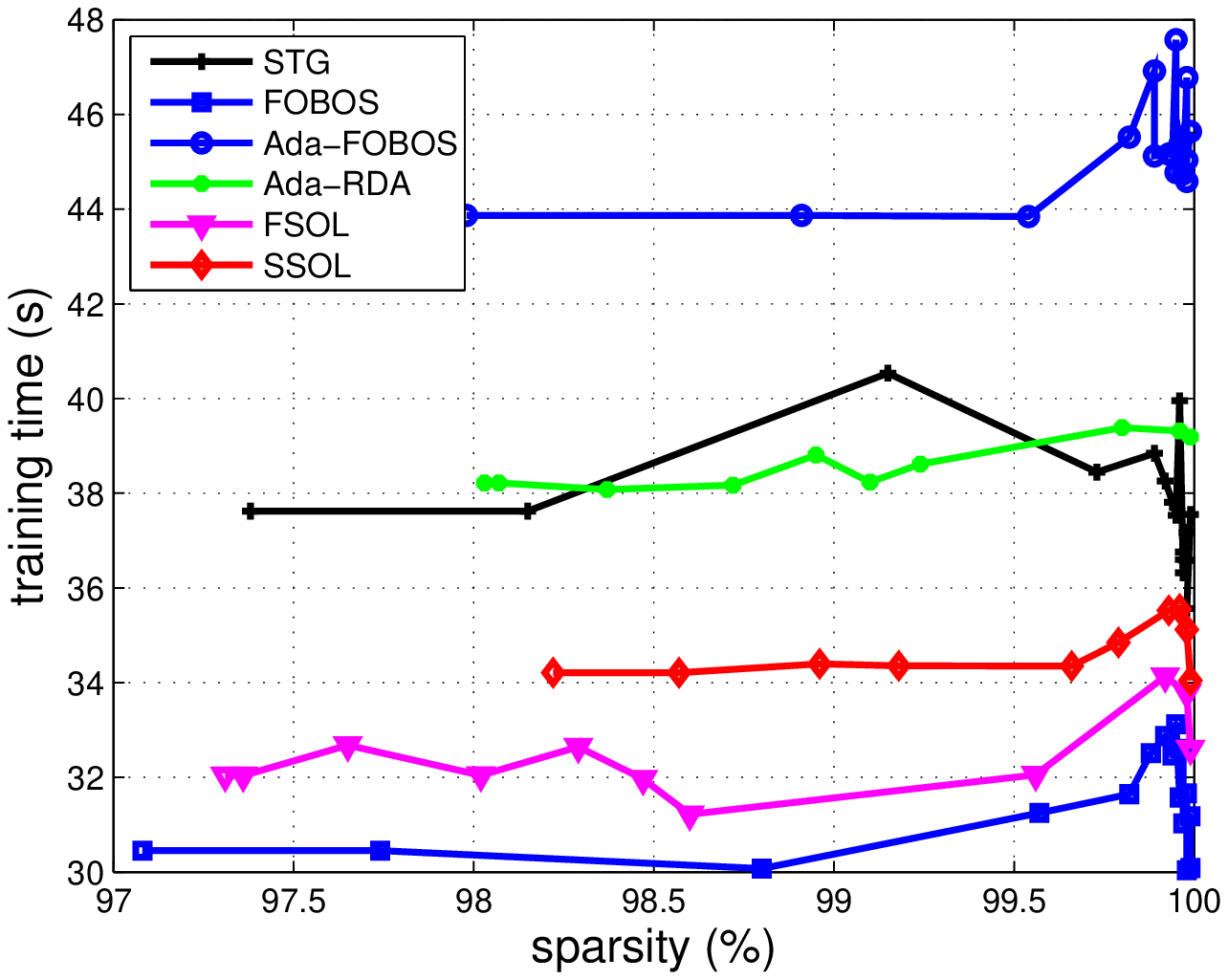} \\
\makeMyboxB{(c) URL}
\makeMyboxB{(d) WEBSPAM} \\
\caption{Time cost on four large-scale datasets: NEWS, RCV1, URL, and WEBSPAM}\label{fig:running_time}
\end{figure*}

\subsection{Applications on Online Anomaly Detection}
Our last two experiments are to explore the proposed sparse online classification technique with application to an online anomaly detection task, i.e., malicious URL detection and web spam detection, where the class distribution is imbalanced in real-world scenarios.

\subsubsection{Malicious URL Detection}\label{sec:url}
In this experiment, we evaluate the cost-sensitive based online learning algorithms for malicious URL detection task with the benchmark dataset that can be downloaded from~\footnote{\url{http://sysnet.ucsd.edu/projects/url/}}. The original URL data set is created in purpose to make it somehow class-balanced, and it has already been used in some previous studies.

In this experiment, we create a subset (denoted as ``ULR2") by sampling from the original data set to make it close to a more realistic distribution scenario where the number of normal URLs is significantly larger than the number of malicious URLs. Following the experiment setting in~\cite{Zhao:2013:COA:2487575.2487647}, we choose $10,000$ positive (malicious) instances and $990,000$ negative (normal) instance. Hence, the ratio $T_+ \setminus T_- = 1 \setminus 99$. For test dataset, we collect $100,000$ samples from the original test set with the same ratio. More details of the unbalanced URL dataset are shown in Table~\ref{tab:datasets}.
\begin{figure}[h]
\centering
\includegraphics[width=3.2in]{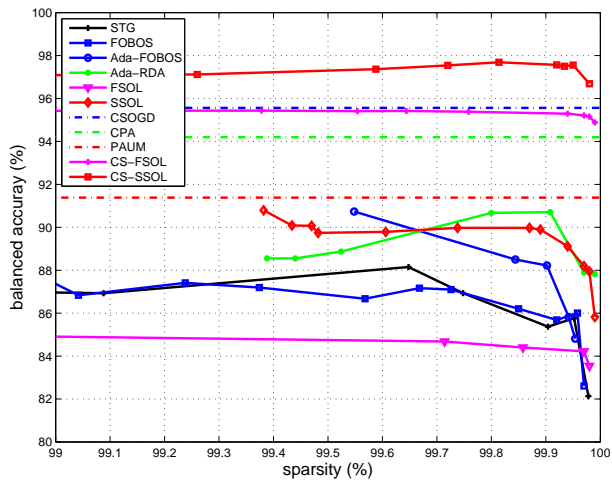}
\caption{Balanced accuracy of different algorithms for malicious URL detection.}\label{fig:result_url}
\end{figure}

We compare the proposed CS-FSOL and CS-SSOL with three other cost-sensitive algorithms (CS-OGD, CPA, and PAUM), as shown in Table~\ref{tab:algs}. In addition, we compare all the cost-insensitive based algorithms to evaluate the classification accuracy without adopting the cost-sensitive lost function. The experiment results are shown in Figure~\ref{fig:result_url}, where CS-OGD, CPA, and PAUM are non-sparse online learning algorithms and thus are invariant to the sparsity.

Several observations can be drawn from the results. First of all, all the cost-sensitive algorithms perform consistently better than their cost-insensitive versions. This indicates that the proposed cost-sensitive algorithm with cost-sensitive loss functions is able to effectively resolve the class-imbalance problem. Second, among all the cost-insensitive algorithms, the second order online learning algorithms are generally better than the first order algorithms. Third, among all the compared algorithms, the proposed CS-SSOL algorithm achieves the best performance, which again validates the efficacy of the proposed technique for real-world data stream classification applications.

\subsubsection{Web Spam Detection}
\begin{figure}[h]
\centering
\includegraphics[width=3.2in]{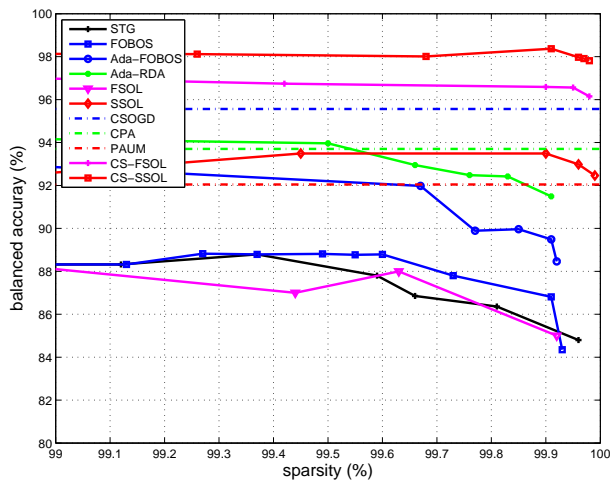}
\caption{Balanced accuracy of different algorithms for web spam detection.}\label{fig:result_web}
\end{figure}

In this experiment, we evaluate the proposed cost-sensitive based online learning algorithms for web spam detection task. We constructed an unbalanced subset of the original web spam dataset used in Section~\ref{sec:large_dataset}.
In particular, for the train dataset, we randomly choose $1,000$ positive instances and $99,000$ negative instances. Hence, the ratio $T_+ \setminus T_-$ of the training set is $1 \setminus 99$. For test dataset, we collect $10,000$ samples from the original test set with the same positive-negative ratio.

We denote the imbalance web spam dataset as ``WEBSPAM2". More details of the unbalanced web spam dataset are shown in Table~\ref{tab:datasets}. As we can see, the feature dimension of WEBSPAM2 dataset ($16,071,971$) is much higher than the one of URL2 ($3,231,961$), and feature representations of WEBSPAM2 dataset are extremely sparse (96.19\% versus 44.96\%). Hence, the anormaly detection task on WEBSPAM2 dataset is very challenge with high-dimensional sparse features and unbalanced data distributions. The experiment settings in this section are the same with Section~\ref{sec:url}, where all cost-sensitive and cost-insensitive algorithms are compared. The experiment results are shown in Figure~\ref{fig:result_web}.

Several observations can be drawn from the results. First of all, for this sparse classification problem, the performances of non-sparse cost-sensitive algorithms decrease significantly. In particular, the cost-insensitive algorithms SSOL and Ada-RDA outperform the cost-sensitive algorithm CAP and PAUM. Second, similar to the previous experiment, the second order online learning algorithms are generally better than the first order algorithms among all the cost-insensitive / cost-sensitive algorithms. Third, the proposed CS-SSOL algorithm consistently achieves the best performance, which again validates the efficacy of the proposed technique for real-world data stream classification applications.

\section{Conclusions and Future work}
In this paper we introduced a framework of sparse online classification (SOC) for large-scale high-dimensional data stream classification tasks. We first showed that the framework essentially includes an existing first-order sparse online classification algorithm as a special case, and can be further extended to derive new sparse online classification algorithms by exploiting second-order information. We also extend the proposed technique to solve cost-sensitive data stream classification problems and explore its applications to online anomaly detection tasks: \emph{malicious URL detection} and \emph{web spam detection}. We analyzed the performance of the proposed algorithms with both theoretical analysis and empirical studies, in which our encouraging experimental results showed that
the proposed algorithms are able to achieve the state-of-the-art performance in comparison to a large family of diverse online learning algorithms.


\end{document}